\documentclass[letterpaper]{article} 
\usepackage{aaai}  
\nocopyright
\usepackage{times}  
\usepackage{helvet}  
\usepackage{courier}  

\usepackage{graphicx}

\usepackage{natbib}  
\usepackage{caption} 
\usepackage{algorithm}
\usepackage{algorithmic}
\usepackage{amsmath}
\usepackage{bm}
\usepackage{amsthm}
\usepackage{booktabs}       
\usepackage{amsfonts}       
\usepackage{nicefrac}       
\usepackage{microtype}      
\usepackage{xcolor}         
\usepackage{mathtools}

\usepackage{siunitx}  
\usepackage{multirow}

\usepackage{times}
\usepackage{helvet}
\usepackage{courier}
\usepackage{xcolor}

\usepackage{amsmath}
\usepackage{amssymb}
\usepackage{mathtools}
\usepackage{amsthm}

\usepackage[utf8]{inputenc}

\usepackage{url}            
\usepackage{booktabs}       
\usepackage{amsfonts}       
\usepackage{nicefrac}       
\usepackage{microtype}      
\usepackage{xcolor}

\usepackage{amsmath,amsfonts,bm}

\def\eqref#1{Eq.~(\ref{#1})}

\def\1{\bm{1}}

\DeclareMathAlphabet{\mathsfit}{\encodingdefault}{\sfdefault}{m}{sl}
\SetMathAlphabet{\mathsfit}{bold}{\encodingdefault}{\sfdefault}{bx}{n}

\usepackage{xcolor}         

\usepackage{algorithm}

\usepackage{algorithmic}
\usepackage{amsmath}
\usepackage{amssymb}
\usepackage{mathtools}
\usepackage{amsthm}
\usepackage{multicol}

\usepackage{caption}
\usepackage{graphicx}

\usepackage{enumitem}
\usepackage{subfigure}
\usepackage{bm}
\usepackage{array}
\usepackage{tabularx}
\usepackage{colortbl}
\usepackage{indentfirst}

\usepackage{lipsum}		
\usepackage{natbib}
\usepackage{appendix}

\usepackage{import}
\usepackage{bbding}
\usepackage{threeparttable}
\usepackage{dsfont}
\usepackage{mathrsfs}

\usepackage{graphicx}
\usepackage{multirow}
\usepackage{makecell}
\usepackage{tcolorbox}

\usepackage{minted} 
\usepackage{xcolor} 
\definecolor{bg}{rgb}{0.95,0.95,0.92}

\newcommand{\eg}{{\em e.g.}}

\theoremstyle{plain}
\newtheorem{theorem}{Theorem}[section]

\theoremstyle{definition}

\theoremstyle{remark}

\usepackage{amsmath}
\usepackage{amssymb}

\usepackage{booktabs}

\RequirePackage{etex} 
\usepackage{cleveref}

\usepackage{listings}
\usepackage{mathtools}
\usepackage{autonum}
\usepackage{bibunits}

\usepackage{newfloat}
\usepackage{listings}
\DeclareCaptionStyle{ruled}{labelfont=normalfont,labelsep=colon,strut=off} 
\floatstyle{ruled}
\newfloat{listing}{tb}{lst}{}
\floatname{listing}{Listing}

\title{Unbiased Top-$k$ Estimation for On-Policy Distillation} 

\author{
    Linjian Meng\equalcontrib\textsuperscript{\rm 1},
    Siyuan Gan\equalcontrib\textsuperscript{\rm 2},
    YuHan Li\textsuperscript{\rm 1},
    Xiran Wang\textsuperscript{\rm 1},\\
    Ziyang Ding\textsuperscript{\rm 1},
    Zitang Gou\textsuperscript{\rm 1},
    Yiming Wu\textsuperscript{\rm 1},
    Zhen Zhao\textsuperscript{\rm 1}\thanks{Corresponding author.}
}
\affiliations{
    \textsuperscript{\rm 1}Shanghai Artificial Intelligence Laboratory, Shanghai, China\\
    \textsuperscript{\rm 2}State Key Laboratory of Novel Software Technology, Nanjing University, Nanjing, China\\
    menglinjian@pjlab.org.cn, gansiyuan@smail.nju.edu.cn, zhaozhen@pjlab.org.cn
}

\usepackage{bibentry}

\begin{document}

\maketitle
\flushbottom 

\begin{abstract}
On-policy distillation (OPD) is becoming an important component of large language model (LLM) post-training for transferring the reasoning capability of a strong teacher LLM to a weaker student LLM.
OPD trains the student by minimizing the reverse KL divergence between the teacher and the student via rollouts generated by the student's policy.
However, estimating the gradient of the reverse KL divergence in OPD remains a challenge. 
Using only the sampled token from the student-generated rollout is computationally cheap but provides limited distributional supervision, which will degrade accuracy. 
In addition, using the full vocabulary provides complete distributional supervision but is computationally expensive.
Therefore, recent works propose Top-$k$ OPD (TK-OPD) that use selected top-$k$ tokens, which provides richer distributional supervision than sampled-token estimation at substantially lower computational cost than full-vocabulary estimation.
Unfortunately, using only the selected top-$k$ tokens induces bias, leading to accuracy degradation, as the probability mass outside the selected top-$k$ tokens is discarded.
To address the bias of TK-OPD, we propose \textit{Tail-Corrected Top-$k$ On-Policy Distillation} (TT-OPD).
It preserves the advantages of TK-OPD, including rich distributional supervision and low computational cost, while providing an unbiased estimator of the gradient of the reverse KL divergence.
The key insight of TT-OPD is to use not only the selected top-$k$ tokens, but also the sampled token from the student-generated rollout, thereby recovering the discarded probability mass in expectation, avoiding the bias.
Experimental results demonstrate that TT-OPD significantly outperforms other tested OPD variants.
\end{abstract}

\section{Introduction}\label{sec:Introduction}

\begin{figure}[!t]
    \centering
    \includegraphics[width=0.88\columnwidth]{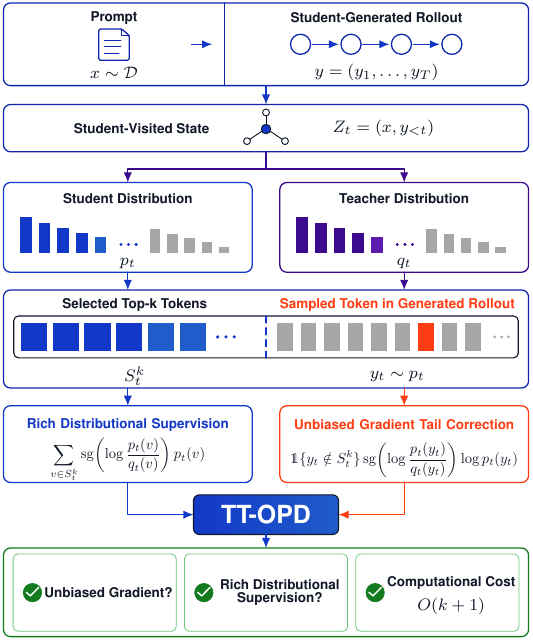}
    \vspace{-0.15cm}
    \caption{\textbf{Overview of TT-OPD.} The notation $sg(\cdot)$ denotes the stop-gradient operator. At each student-visited state $Z_t$, the selected top-$k$ tokens $S_t^k$ provide rich distributional supervision, while the sampled token $y_t$ in the student-generated rollout provides an unbiased gradient tail correction. This only induces $O(k+1)$ computational cost.}
    \label{fig:tt-opd-overview}
    \vspace{-0.2cm}
\end{figure}

On-policy distillation (OPD)~\citep{agarwal2024policy,Lu2025thinkinglab} is becoming an important component of large language model (LLM) post-training, as evidenced by its utilization in various frontier-level LLMs~\citep{yang2025qwen3,zeng2026glm,xu2026deepseek,xiao2026mimo}. This is due to OPD's ability to transfer the reasoning capability of a strong teacher LLM to a weaker student LLM~\citep{gu2024minillm,xu2025speculative,fu2026revisiting,oh2026kl,xiao2026mimo,zheng2026scope,zhu2026many,zhao2026self,yang2026learning,hubotter2026reinforcement,shenfeld2026self,yang2026learning,wu2026lightning,hou2026uni,wang2026teachability,xie2026position}.
Specifically, unlike off-policy distillation, which trains the student on teacher-generated rollouts, OPD minimizes the reverse KL divergence between the teacher and the student using rollouts generated by the student's own policy. This on-policy setting enables the student to refine its behavior on states it actually visits~\citep{li2026rethinking}. It reduces the substantial exposure bias inherent in off-policy distillation that hinders the transfer of reasoning capability from a strong teacher LLM to a weaker student LLM~\citep{song2026survey,hou2026uni}.

\begin{table*}[t]
\centering
\begingroup
\renewcommand{\arraystretch}{1.1}
\renewcommand{\tabularxcolumn}[1]{m{#1}}
\begin{tabularx}{0.95\textwidth}{>{\hsize=0.7\hsize\centering\arraybackslash}X>{\hsize=1.0\hsize\centering\arraybackslash}X>{\hsize=1.35\hsize\centering\arraybackslash}X>{\hsize=0.95\hsize\centering\arraybackslash}X}
\toprule
\textbf{  } & \textbf{Unbiased Gradient?} & {\small\textbf{Rich Distributional Supervision?}} & \textbf{Computational Cost} \\
\midrule
ST-OPD & \textcolor{green!60!black}{$\checkmark$} & \textcolor{red}{$\times$} & $O(1)$ \\
FV-OPD & \textcolor{green!60!black}{$\checkmark$} & \textcolor{green!60!black}{$\checkmark$} & $O(|\mathcal{V}|)$ \\
TK-OPD & \textcolor{red}{$\times$} & \textcolor{green!60!black}{$\checkmark$} & $O(k)$ \\
\rowcolor{blue!8}
TT-OPD & \textcolor{green!60!black}{$\checkmark$} & \textcolor{green!60!black}{$\checkmark$} & $O(k+1)$ \\
\bottomrule
\end{tabularx}
\endgroup
\caption{Comparison between our TT-OPD and other OPD variants. The notation $\mathcal{V}$ denotes the vocabulary.}
\label{tab:opd-variants}
\vspace{-0.35cm}
\end{table*}

Unfortunately, estimating the gradient of the reverse KL divergence in OPD remains a challenge. Vanilla OPD, also called Sampled-Token OPD (ST-OPD), uses only the sampled token in the student-generated rollout to estimate the gradient of the reverse KL divergence~\citep{Lu2025thinkinglab}. While this approach incurs low computational cost, it provides limited distributional supervision, which will degrade accuracy. To overcome this limitation, some works propose Full-Vocabulary OPD (FV-OPD) that computes the gradient of the reverse KL divergence using the full vocabulary~\citep{xu2026deepseek}. However, this incurs prohibitive computational cost. For example, Qwen3~\citep{yang2025qwen3} has a vocabulary of $151,646$ tokens, so FV-OPD evaluates $151,646$ terms at each position, inducing $151,646$ times larger computational cost than ST-OPD. To trade off the richness of distributional supervision and computational cost, recent works introduce Top-$k$ OPD (TK-OPD) that adopts selected top-$k$ tokens to provide rich  distributional supervision~\citep{hubotter2026reinforcement,li2026rethinking}, where $k$ is much smaller than the vocabulary size. However, this induces bias as the probability mass outside the selected top-$k$ tokens is discarded~\citep{zhu2026many}, leading to accuracy degradation.

To address the bias of TK-OPD, we propose a novel OPD variant called \textit{Tail-Corrected Top-$k$ On-Policy Distillation} (TT-OPD).
It preserves the advantages of TK-OPD, including rich distributional supervision and low computational cost, while providing an unbiased estimator of the gradient of the reverse KL divergence.
As shown in \Cref{fig:tt-opd-overview}, the key insight of TT-OPD is to use not only the selected top-$k$ tokens, but also the sampled token in the student-generated rollout.
Specifically, since the sampled token is drawn from the full student distribution, it provides a tail correction to recover the probability mass outside the selected top-$k$ tokens in expectation, thereby avoiding the bias of TK-OPD. 
More importantly, TT-OPD also retains all advantages of TK-OPD, \eg, rich distributional supervision and low computational cost. 
To more clearly illustrate the differences between TT-OPD and other OPD variants, we present Table~\ref{tab:opd-variants}.

To evaluate TT-OPD, we employ the Qwen3~\citep{yang2025qwen3} and DeepSeek~\citep{guo2025deepseekr1} model series and conduct experiments on mathematical reasoning benchmarks and code-generation benchmarks. Experimental results demonstrate that TT-OPD yields significant accuracy improvements over ST-OPD and TK-OPD across student scales, model series, selected top-$k$ tokens, values of $k$, and tasks.\footnote{Due to the prohibitive computational cost of FV-OPD, we are unable to compare with it.}

\section{Related Work}\label{sec:Related Work}

\textbf{Off-Policy Distillation.}
A major challenge in LLM post-training is transferring the reasoning capability of a strong teacher LLM to a weaker student LLM. The initial approach, off-policy distillation, trains the student using teacher-generated rollouts. 
One line of off-policy distillation algorithm aligns the student's distribution with that of the teacher on a static dataset, typically through a forward-KL objective~\citep{hinton2015distilling,xu2024surveyllmkd,liu2024ddk,guo2025learningfocus,deng2026less,ko2025distillm2}.
Another line of off-policy distillation algorithms trains the student on teacher-generated tokens, as in supervised fine-tuning (SFT) style imitation~\citep{wang2022selfinstruct,guo2025deepseekr1,zhang2025openmmreasoner,cha2024honeybee,zhang2025bee}. 
However, because off-policy distillation uses states from teacher-generated rollouts, these states may differ from those the student actually visits, inducing substantial exposure bias that hinders reasoning capability transfer~\citep{song2026survey,hou2026uni}.

\textbf{On-Policy Distillation.}
To reduce the exposure bias in off-policy distillation, recent works adopt on-policy distillation (OPD)~\citep{agarwal2024policy,Lu2025thinkinglab}, as exemplified by various frontier-level LLMs~\citep{yang2025qwen3,zeng2026glm,xu2026deepseek,xiao2026mimo}. OPD trains the student through student-generated rollouts, ensuring that the training states match those the student actually visits to reduce the exposure bias. Specifically, after sampling rollouts from the student's policy, OPD trains the student by minimizing the reverse KL divergence at each position of the student-generated rollouts (see more details in \Cref{sec:Preliminaries}).

\textbf{Gradient Estimation for the Reverse KL Divergence in OPD.}
A central challenge in OPD is estimating the gradient of the reverse KL divergence under the large vocabulary of modern LLMs. Sampled-Token OPD (ST-OPD) uses only the sampled token, achieving low computational cost, but providing limited distributional supervision, which degrades accuracy~\citep{Lu2025thinkinglab,oh2026kl}. Full-Vocabulary OPD (FV-OPD) uses full-vocabulary distributional supervision by summing all tokens over the full vocabulary, but its computational cost scales with vocabulary size, making it expensive for modern large-vocabulary LLMs~\citep{xu2026deepseek}. Top-$k$ OPD (TK-OPD) trades off the richness of distributional supervision and computational cost by using the selected top-$k$ tokens to estimate the gradient of the reverse KL divergence~\citep{hubotter2026reinforcement,li2026rethinking}. However, discarding probability mass outside the selected top-$k$ tokens introduces bias~\citep{zhu2026many}.

In contrast to these three OPD variants, our TT-OPD retains all advantages of TK-OPD, rich distributional supervision and low computational cost independent of vocabulary size, while also inheriting the unbiasedness of the estimated gradient of ST-OPD and FV-OPD. Consequently, TT-OPD is as suitable for modern large-vocabulary LLMs as ST-OPD and TK-OPD, yet significantly outperforms both.

\section{Preliminaries}\label{sec:Preliminaries}

\subsection{OPD Framework}\label{subsec:OPD Framework}

We first formalize the OPD framework used in this paper. Let $\pi_\theta$ denote the student policy and $\pi_{\mathrm{tea}}$ denote the fixed teacher policy over a shared vocabulary $\mathcal{V}$. Given a prompt $x\sim\mathcal{D}$, OPD samples a rollout from the student's policy:
\begin{equation}
\label{eq:prelim-student-rollout}
\begin{aligned}
    y=(y_1,\ldots,y_T)\sim\pi_\theta(\cdot\mid x),
    \qquad T=|y|,
\end{aligned}
\end{equation}
where $T$ is the realized rollout length. We refer to $y$ as the student-generated rollout since it is sampled via student's policy $\pi_\theta$. For each position $t\in\{1,\ldots,T\}$, let $Z_t=(x,y_{<t})$ denote the student-visited state before token $y_t$. The student and teacher next-token distributions at $Z_t$ are
\begin{equation}
\label{eq:prelim-token-distributions}
\begin{aligned}
    p_t(v) &\triangleq \pi_\theta(v\mid Z_t),\\
    q_t(v) &\triangleq \pi_{\mathrm{tea}}(v\mid Z_t),
    \qquad v\in\mathcal{V}.
\end{aligned}
\end{equation}

At position $t$, the full-vocabulary reverse KL divergence is
\begin{equation}
\label{eq:prelim-token-kl}
\begin{aligned}
    d_t(\theta)
    &\triangleq D_{\mathrm{KL}}(p_t\|q_t)
    = \sum_{v\in\mathcal{V}} p_t(v)\log\frac{p_t(v)}{q_t(v)}.
\end{aligned}
\end{equation}
In practical training, OPD constructs a single-sample empirical objective over the finite training set. Specifically, for each prompt $x\in\mathcal{D}$, it samples one rollout
\begin{equation}
\label{eq:prelim-dataset-rollouts}
\begin{aligned}
    y^{(x)}=(y^{(x)}_1,\ldots,y^{(x)}_{T_x})
    \sim\pi_\theta(\cdot\mid x),
    \qquad T_x=|y^{(x)}|.
\end{aligned}
\end{equation}
Let $d_{x,t}(\theta)$ denote the divergence defined above, evaluated at $Z_{x,t}=(x,y^{(x)}_{<t})$. Using token-mean aggregation, the resulting OPD objective is
\begin{equation}
\label{eq:prelim-opd-objective}
\begin{aligned}
    \mathcal{L}_{\mathrm{OPD}}(\theta)
    \triangleq
    \frac{1}{\sum_{x\in\mathcal{D}}T_x}
    \sum_{x\in\mathcal{D}}
    \sum_{t=1}^{T_x} d_{x,t}(\theta).
\end{aligned}
\end{equation}

Following standard OPD training, the sampled rollouts, their realized lengths, and the visited states are treated as fixed during backpropagation; gradients do not propagate through the rollout-sampling process. Let $g_{x,t}^\star\triangleq\nabla_\theta d_{x,t}(\theta)$. The corresponding training gradient targeted by all estimators in this paper is
\begin{equation}
\label{eq:prelim-opd-gradient}
\begin{aligned}
    g_{\mathrm{OPD}}(\theta)
    \triangleq \nabla_\theta\mathcal{L}_{\mathrm{OPD}}(\theta)
    = \frac{1}{\sum_{x\in\mathcal{D}}T_x}
    \sum_{x\in\mathcal{D}}
    \sum_{t=1}^{T_x} g_{x,t}^\star.
\end{aligned}
\end{equation}
When discussing a generic rollout below, we omit the prompt index and write $g_t^\star$ for $g_{x,t}^\star$.
Unlike off-policy distillation algorithms, OPD evaluates this token-level objective on states visited by the student's own policy. This on-policy setting reduces the exposure bias caused by training on teacher-generated data~\citep{li2026rethinking,song2026survey}.

\subsection{Three OPD Variants}\label{Three OPD Variants}

One of the computational bottlenecks of the OPD training lies in estimating $g_t^\star$ for modern large-vocabulary LLMs. Existing OPD variants mainly differ in the number of vocabulary tokens used to form this estimate, leading to different trade-offs among bias, richness of distributional supervision, and computational cost.

\textbf{Sampled-Token OPD (ST-OPD).}
Vanilla OPD, also called ST-OPD, uses only the token $y_t$ sampled from the student-generated rollout~\citep{Lu2025thinkinglab}. Since $y_t\sim p_t$ at a fixed state $Z_t$, its stop-gradient surrogate is
\begin{equation}
\label{eq:prelim-sampled-token-opd}
\begin{aligned}
    \mathcal{L}^{\mathrm{ST}}_t(\theta)
    = \operatorname{sg}\!\left(
        \log \frac{p_t(y_t)}{q_t(y_t)}
      \right)
      \log p_t(y_t),
\end{aligned}
\end{equation}
where $\operatorname{sg}(\cdot)$ denotes the stop-gradient operator. Its gradient is an unbiased estimator of the gradient of the reverse KL divergence~\citep{oh2026kl}:
\begin{equation}
\label{eq:prelim-sampled-token-equal}
\begin{aligned}
    \mathbb{E}_{y_t\sim p_t}
    \left[
        \nabla_\theta\mathcal{L}^{\mathrm{ST}}_t(\theta)
        \mid Z_t
    \right]
    = g_t^\star.
\end{aligned}
\end{equation}
ST-OPD therefore has low computational cost, but its reliance on one sampled token provides limited distributional supervision, which degrades accuracy.

\textbf{Full-Vocabulary OPD (FV-OPD).}
FV-OPD evaluates the reverse KL divergence over the entire vocabulary~\citep{xu2026deepseek}:
\begin{equation}
\label{eq:prelim-full-vocabulary-opd}
\begin{aligned}
    \mathcal{L}^{\mathrm{FV}}_t(\theta)
    = D_{\mathrm{KL}}(p_t\|q_t)
    = \sum_{v\in\mathcal{V}} p_t(v)\log\frac{p_t(v)}{q_t(v)}.
\end{aligned}
\end{equation}
Consequently,
\begin{equation}
\label{eq:prelim-full-vocabulary-opd-equal}
\begin{aligned}
    \nabla_\theta\mathcal{L}^{\mathrm{FV}}_t(\theta)
    = g_t^\star.
\end{aligned}
\end{equation}
This estimator provides full-vocabulary distributional supervision, but it incurs prohibitive computational cost because it evaluates token-level terms over the full vocabulary.

\textbf{Top-$k$ OPD (TK-OPD).}
To trade off the richness of distributional supervision and computational cost, TK-OPD restricts the reverse KL computation divergence to a selected subset of the full vocabulary~\citep{hubotter2026reinforcement,li2026rethinking,oh2026kl,zhu2026many}. Let $S_t^k$ denote the selected top-$k$ tokens at position $t$ that can be the student's top-$k$ tokens $S_{t,\mathrm{stu}}^k$, the teacher's top-$k$ tokens $S_{t,\mathrm{tea}}^k$, or their overlap $S_{t,\mathrm{stu}}^k\cap S_{t,\mathrm{tea}}^k$. For every $v\in S_t^k$, TK-OPD renormalizes the student and teacher distributions within $S_t^k$:
\begin{equation}
\label{eq:prelim-topk-normalization}
\begin{aligned}
    p_t^{S_t^k}(v) &\triangleq
    \frac{p_t(v)}{\sum_{u\in S_t^k}p_t(u)},\\
    q_t^{S_t^k}(v) &\triangleq
    \frac{q_t(v)}{\sum_{u\in S_t^k}q_t(u)}.
\end{aligned}
\end{equation}
Its per-position loss is
\begin{equation}
\label{eq:prelim-topk-opd}
\begin{aligned}
    \mathcal{L}^{\mathrm{TK}}_t(\theta)
    &= D_{\mathrm{KL}}(p_t^{S_t^k}\|q_t^{S_t^k})\\
    &= \sum_{v\in S_t^k}p_t^{S_t^k}(v)
    \log\frac{p_t^{S_t^k}(v)}{q_t^{S_t^k}(v)}.
\end{aligned}
\end{equation}
TK-OPD provides richer distributional supervision than ST-OPD and its computational cost is much cheaper than FV-OPD because $1\ll k\ll|\mathcal{V}|$. However, TK-OPD introduces bias because the probability mass outside the selected top-$k$ tokens is discarded~\citep{zhu2026many}:
\begin{equation}
\label{eq:prelim-topk-opd-noequal}
\begin{aligned}
    \nabla_\theta\mathcal{L}^{\mathrm{TK}}_t(\theta)
    \neq g_t^\star.
\end{aligned}
\end{equation}
Thus, although TK-OPD provides a trade-off between the richness of distributional supervision and computational cost, it cannot provide an unbiased estimator of the gradient of the reverse KL divergence, leading to accuracy degradation.

\section{Methodology}\label{sec:our_methods}

To mitigate the bias of estimating the gradient in TK-OPD, we propose \textit{Tail-Corrected Top-$k$ On-Policy Distillation} (TT-OPD). Its key insight is to use not only the selected top-$k$ tokens, but also the sampled token in the student-generated rollout. Specifically, it retains the tokens $S_t^k$ used by TK-OPD but also reuses the token $y_t$ sampled in the student-generated rollout $y$. Because $y_t$ is drawn from the full student distribution, it provides a tail correction to recover the discarded probability mass outside the selected top-$k$ tokens in expectation. In other words, the selected top-$k$ tokens provide rich distributional supervision, while the sampled token $y_t$ provides an unbiased correction for the discarded tail contribution. More importantly, TT-OPD incurs a vocabulary-independent computational cost of only $O(k+1)$, which is nearly identical to the $O(k)$ computational cost of TK-OPD and substantially lower than the $O(|\mathcal{V}|)$ computational cost of FV-OPD.

\subsection{Overview of TT-OPD}\label{subsec:tt-opd-overview}

TT-OPD uses the on-policy token distributions defined in \Cref{subsec:OPD Framework}. At a student-visited state $Z_t=(x,y_{<t})$, the student distribution is $p_t$, the teacher distribution is $q_t$, and the sampled token $y_t$ satisfies $y_t\sim p_t$. Let $S_t^k$ denote the same selected tokens used by TK-OPD.

The selected set $S_t^k$ specifies which token-level gradient contributions are computed exactly, while the sampled token $y_t$ provides the tail correction. TT-OPD defines the following per-position loss:
\begin{equation}
\label{eq:tt-opd-loss}
\begin{aligned}
    \mathcal{L}^{\mathrm{TT}}_t(\theta)
    =
    &\sum_{v\in S_t^k}
    \operatorname{sg}\!\left(
        \log\frac{p_t(v)}{q_t(v)}
    \right)p_t(v)\\
    &+\mathds{1}\{y_t\notin S_t^k\}
    \operatorname{sg}\!\left(
        \log\frac{p_t(y_t)}{q_t(y_t)}
    \right)\log p_t(y_t).
\end{aligned}
\end{equation}
The first term exactly computes the contributions of tokens in $S_t^k$ to the gradient of the reverse KL divergence. The second term estimates the omitted contribution from the tail $\mathcal{V}\setminus S_t^k$. When $y_t\in S_t^k$, this tail correction is omitted because the sampled token is already covered by the first term. The computational cost of TT-OPD at each position is at most $O(|S_t^k|+1)$, which is $O(k+1)$.

Now, we provide a theoretical analysis to demonstrate that TT-OPD provides an unbiased estimator of the gradient of the reverse KL divergence, as shown in \Cref{thm:tt-opd-unbiased}.

\begin{theorem}\label{thm:tt-opd-unbiased}
[Proof is in \Cref{sec:proof-tt-opd-unbiased}.]
At any fixed $Z_t=(x,y_{<t})$, the TT-OPD loss in Eq.~(\ref{eq:tt-opd-loss}) satisfies
\begin{equation}
\label{eq:tt-opd-unbiased-gradient}
\begin{aligned}
    \mathbb{E}_{y_t\sim p_t}
    \left[
        \nabla_\theta\mathcal{L}^{\mathrm{TT}}_t(\theta)
        \mid Z_t
    \right]
    = g_t^\star.
\end{aligned}
\end{equation}
\end{theorem}

\subsection{Discussion of TT-OPD}\label{subsec:Discussion of TT-OPD}

Compared with ST-OPD, FV-OPD, and TK-OPD, TT-OPD combines their complementary advantages without inheriting their main limitations. Firstly, like ST-OPD and FV-OPD, TT-OPD provides an unbiased estimator of the gradient of the reverse KL divergence. However, ST-OPD provides limited distributional supervision, which can degrade accuracy, while FV-OPD incurs a prohibitive computational cost of $O(|\mathcal{V}|)$. In contrast, TT-OPD provides rich distributional supervision at a vocabulary-independent cost of only $O(k+1)$. Secondly, TT-OPD retains the rich distributional supervision and low computational cost of TK-OPD, while avoiding its gradient bias. Consequently, TT-OPD is as suitable for modern large-vocabulary LLMs as ST-OPD and TK-OPD, while significantly outperforming both in our experiments.

Additionally, the core of TT-OPD is to replace only the per-position OPD loss. This makes TT-OPD compatible with improvements to other components of the OPD pipeline. For example, TT-OPD can be combined with OPD variants that avoid sampling complete rollouts through fixed, progressive, truncated, or adaptive rollout horizons. These include Early Stopping Rollout (ESR)~\citep{zhou2026earlystopping}, Progressive OPD (POPD), Truncated OPD (TOPD)~\citep{zhang2026fullrollouts}, and KAT~\citep{xin2026escaping}. Under these combinations, the TT-OPD loss is computed only at the retained positions. Moreover, TT-OPD can be combined with OPD variants that reweight token positions. Examples include Teachability-Aware OPD (TA-OPD)~\citep{wang2026teachability} and Importance-Weighted OPD (IW-OPD)~\citep{xie2026position}. These variants determine the position weights, whereas TT-OPD determines how the gradient of the reverse-KL divergence is estimated. These examples are not exhaustive; due to space limitations, we do not enumerate all compatible OPD improvements.

\section{Experiments}\label{sec:experiments}

\providecommand{\resulttodo}{\textcolor{red}{TBD}}

\newcounter{mycounter}
\newcommand\showmycounter{\stepcounter{mycounter}\themycounter}
\newcommand{\findingbox}[2][]{
    \refstepcounter{mycounter}
    \if\relax\detokenize{#1}\relax\else\label{#1}\fi
    \begin{tcolorbox}[colframe=black,
                      arc=1pt,
                      boxsep=-2pt,
                      before skip=5pt,
                      after skip=5pt,
                      ]
        \noindent{\textbf{\textit{Finding \themycounter.}}} #2
    \end{tcolorbox}
}

\subsection{Experimental Setup}\label{subssec:Experimental Setup}

\textbf{Models and Training Data.} Unless otherwise stated, we employ the Qwen3 model series~\citep{yang2025qwen3}. Specifically, the teacher is Qwen3-8B-Base-GRPO-Math, which is obtained by training Qwen3-8B-Base with GRPO on DAPO-Math-17K~\citep{li2026rethinking}. The students are Qwen3-4B-Base and Qwen3-1.7B-Base. All models used in this paper have thinking mode enabled to achieve the highest possible accuracy. Unless otherwise stated, all OPD variants tested in this paper are trained on DAPO-Math-17K for two epochs.

\textbf{Baselines and Training Configuration.} Unless otherwise stated, we compare TT-OPD against ST-OPD and TK-OPD. FV-OPD is omitted as using the entire vocabulary at every position is prohibitively expensive. Unless otherwise stated, TK-OPD and TT-OPD both use the student's top-$16$ set tokens, following the default configuration of \citet{li2026rethinking}. We use a training temperature of $1.0$, batch size of $256$, one rollout per prompt, a maximum prompt length of $2{,}048$, a maximum response length of $8{,}192$, a learning rate of $10^{-6}$, and top-$p=1.0$. Our implementation is based on verl~\citep{sheng2024hybridflow}, and all experiments are run on a server with eight NVIDIA H200 GPUs and $1{,}600$ GB of system memory.

\textbf{Evaluation Configuration.} Unless otherwise stated, we evaluate on seven mathematical reasoning benchmarks: AIME 2024, AIME 2025, AIME 2026, AMC, MATH, Minerva, and Olympiad. We sample 32 responses per prompt for AIME 2024, AIME 2025, AIME 2026, and AMC, and 8 responses per prompt for others. All responses are sampled with temperature $0.6$, top-$p=0.95$, top-$k=20$, and a maximum response length of $32{,}768$ tokens.

\begin{table*}[t]
\centering
\begingroup
\renewcommand{\arraystretch}{1.25}
\setlength{\tabcolsep}{4pt}
\begin{tabularx}{0.95\textwidth}{>{\raggedright\arraybackslash}p{0.14\textwidth}*{8}{>{\centering\arraybackslash}X}}
\toprule
 & \textbf{AIME24} & \textbf{AIME25} & \textbf{AIME26} & \textbf{AMC} & \textbf{MATH} & \textbf{Minerva} & \textbf{Olympiad} & \textbf{Avg.} \\
\midrule
Teacher & 31.15 & 24.69 & 22.71 & 70.63 & 88.02 & 38.10 & 58.91 & 47.74 \\
\midrule
\multicolumn{9}{c}{\textit{Qwen3-4B-Base Student}} \\
Student & 8.02 & 5.31 & 6.15 & 30.80 & 54.07 & 22.79 & 25.22 & 21.77 \\
ST-OPD & 21.04 & 19.17 & 14.27 & 57.83 & 84.67 & 34.33 & 53.44 & 40.68 \\
TK-OPD & 19.79 & 19.17 & 14.27 & 59.04 & 85.23 & 35.29 & 52.31 & 40.73 \\
\rowcolor{blue!8}
\textbf{TT-OPD} & \textbf{26.77} & \textbf{23.33} & \textbf{22.92} & \textbf{64.87} & \textbf{86.08} & \textbf{36.81} & \textbf{54.70} & \textbf{45.07} \\
\midrule
\multicolumn{9}{c}{\textit{Qwen3-1.7B-Base Student}} \\
Student & 3.65 & 1.04 & 2.81 & 25.34 & 52.45 & 15.99 & 22.93 & 17.74 \\
ST-OPD & 9.79 & 2.50 & 2.08 & 37.05 & 66.33 & 20.59 & 30.91 & 24.18 \\
TK-OPD & 8.54 & 2.40 & 4.90 & 35.54 & 68.00 & 22.47 & 31.39 & 24.75 \\
\rowcolor{blue!8}
\textbf{TT-OPD} & \textbf{14.17} & \textbf{6.98} & \textbf{8.33} & \textbf{42.13} & \textbf{69.60} & \textbf{24.63} & \textbf{36.48} & \textbf{28.90} \\
\bottomrule
\end{tabularx}
\endgroup
\caption{\textbf{Main results on mathematical reasoning (accuracy, \%).} Qwen3-8B-Base-GRPO-Math is the teacher for all tested algorithms. All tested algorithms use DAPO-Math-17K as the training set. The student's top-$16$ tokens are used as the selected top-$k$ tokens. \textbf{Bold} = best among the tested algorithms within each student block.}
\label{tab:main-math-results}
\end{table*}

\subsection{Main Results on Mathematical Reasoning}

\findingbox{TT-OPD consistently outperforms ST-OPD and TK-OPD across student scales.}

To evaluate our TT-OPD, we compare it with ST-OPD and TK-OPD on seven mathematical reasoning benchmarks. We report accuracy and pass@$k$ in \Cref{tab:main-math-results} and \Cref{tab:main-math-passk-results} (Appendix~\ref{sec:full-experimental-results}), respectively. In terms of accuracy, TT-OPD achieves the highest average for both students. Compared with ST-OPD, it improves average accuracy by $4.39$ and $4.72$ percentage points for Qwen3-4B-Base and Qwen3-1.7B-Base, respectively. Compared with TK-OPD, the corresponding improvements are $4.34$ and $4.15$ percentage points. Moreover, TT-OPD outperforms both baselines on all seven benchmarks for each student. In terms of pass@$k$, TT-OPD again achieves the highest average for both students. It improves over ST-OPD by $4.20$ and $1.63$ percentage points and over TK-OPD by $7.88$ and $1.73$ percentage points for Qwen3-4B-Base and Qwen3-1.7B-Base, respectively. These results demonstrate the superior performance of TT-OPD over the other OPD variants across student scales.

\subsection{Sensitivity to the Number of Selected Tokens}

\begin{table*}[t]
\centering
\begingroup
\renewcommand{\arraystretch}{1.25}
\setlength{\tabcolsep}{4pt}
\begin{tabularx}{0.95\textwidth}{>{\raggedright\arraybackslash}p{0.14\textwidth}*{8}{>{\centering\arraybackslash}X}}
\toprule
 & \textbf{AIME24} & \textbf{AIME25} & \textbf{AIME26} & \textbf{AMC} & \textbf{MATH} & \textbf{Minerva} & \textbf{Olympiad} & \textbf{Avg.} \\
\midrule
ST-OPD & 21.04 & 19.17 & 14.27 & 57.83 & 84.67 & 34.33 & 53.44 & 40.68 \\
\midrule
\multicolumn{9}{c}{\textit{$k=8$}} \\
TK-OPD & 16.46 & 17.50 & 14.58 & 55.72 & 84.90 & 34.15 & 53.50 & 39.54 \\
\rowcolor{blue!8}
\textbf{TT-OPD} & \textbf{27.60} & \textbf{22.60} & \textbf{21.88} & \textbf{62.88} & \textbf{85.70} & \textbf{36.35} & \textbf{55.96} & \textbf{44.71} \\
\midrule
\multicolumn{9}{c}{\textit{$k=16$}} \\
TK-OPD & 19.79 & 19.17 & 14.27 & 59.04 & 85.23 & 35.29 & 52.31 & 40.73 \\
\rowcolor{blue!8}
\textbf{TT-OPD} & \textbf{26.77} & \textbf{23.33} & \textbf{22.92} & \textbf{64.87} & \textbf{86.08} & \textbf{36.81} & \textbf{54.70} & \textbf{45.07} \\
\midrule
\multicolumn{9}{c}{\textit{$k=32$}} \\
TK-OPD & 20.21 & 18.65 & 16.67 & 58.73 & 85.50 & 35.06 & 52.69 & 41.07 \\
\rowcolor{blue!8}
\textbf{TT-OPD} & \textbf{28.02} & \textbf{24.17} & \textbf{21.98} & \textbf{64.38} & \textbf{85.90} & \textbf{37.27} & \textbf{56.06} & \textbf{45.40} \\
\midrule
\multicolumn{9}{c}{\textit{$k=64$}} \\
TK-OPD & 22.19 & 20.00 & 15.42 & 60.24 & 84.00 & \textbf{37.55} & 53.57 & 41.85 \\
\rowcolor{blue!8}
\textbf{TT-OPD} & \textbf{28.85} & \textbf{22.60} & \textbf{24.79} & \textbf{65.06} & \textbf{86.00} & 37.50 & \textbf{55.20} & \textbf{45.71} \\
\bottomrule
\end{tabularx}
\endgroup
\caption{\textbf{Sensitivity to the number of selected tokens (accuracy, \%).} Qwen3-8B-Base-GRPO-Math is the teacher, Qwen3-4B-Base is the student, and all tested algorithms use DAPO-Math-17K as the training set. The student's top-$16$ tokens are used as the selected top-$k$ tokens. ST-OPD does not depend on $k$. \textbf{Bold} = best among the tested algorithms for each value of $k$.}
\label{tab:k-sensitivity}
\end{table*}

\findingbox{TT-OPD consistently outperforms ST-OPD and TK-OPD across all tested values of $k$.}
To evaluate the sensitivity of TT-OPD to the number of selected tokens, we vary $k$. The accuracy and pass@$k$ results are shown in \Cref{tab:k-sensitivity} and \Cref{tab:k-sensitivity-passk} (Appendix~\ref{sec:full-experimental-results}), respectively. TT-OPD achieves the highest average accuracy and average pass@$k$ for every tested $k$. Across the four values of $k$, TT-OPD improves average accuracy over ST-OPD and TK-OPD by $4.03$--$5.03$ and $3.86$--$5.17$ percentage points, respectively, while improving average pass@$k$ by $2.38$--$8.71$ and $6.97$--$8.71$ percentage points. These results demonstrate that TT-OPD consistently achieves better performance than TK-OPD.

\subsection{Training Efficiency}

\findingbox{TT-OPD increases training time by at most approximately $6\%$ over TK-OPD.}

To evaluate the computational cost of TT-OPD, we compare its training time with TK-OPD for $k\in\{8,16,32,64\}$. The experimental results are shown in \Cref{tab:training-efficiency}. Across the eight matched configurations, TT-OPD adds only 10--15 minutes. The largest relative increase is about $6\%$. This result is consistent with TT-OPD's vocabulary-independent computational cost of $O(k+1)$, which is nearly identical to the $O(k)$ computational cost of TK-OPD. These results demonstrate that TT-OPD introduces only marginal additional computational cost over TK-OPD.

\begin{table}[t]
\centering
\begingroup
\renewcommand{\arraystretch}{1.31}
\setlength{\tabcolsep}{4pt}
\begin{tabularx}{0.95\columnwidth}{>{\raggedright\arraybackslash}X>{\centering\arraybackslash}X>{\columncolor{blue!8}\centering\arraybackslash}X}
\toprule
$k$ & TK-OPD & TT-OPD \\
\midrule
\multicolumn{3}{c}{\textit{Qwen3-4B-Base Student}} \\
$k=8$  & 4 h 22 min & 4 h 34 min \\
$k=16$ & 4 h 27 min & 4 h 38 min \\
$k=32$ & 4 h 35 min & 4 h 45 min \\
$k=64$ & 4 h 38 min & 4 h 53 min \\
\midrule
\multicolumn{3}{c}{\textit{Qwen3-1.7B-Base Student}} \\
$k=8$  & 3 h 47 min & 3 h 56 min \\
$k=16$ & 3 h 53 min & 4 h 03 min \\
$k=32$ & 3 h 58 min & 4 h 09 min \\
$k=64$ & 4 h 05 min & 4 h 19 min \\
\bottomrule
\end{tabularx}
\endgroup
\caption{\textbf{Training time of TK-OPD and TT-OPD.} Qwen3-8B-Base-GRPO-Math is the teacher. Except for $k$ range in $\{8,16,32,64\}$, all other configurations are identical to those described in \Cref{subssec:Experimental Setup}
.}
\label{tab:training-efficiency}
\end{table}

\subsection{Robustness to Different Selected Top-$k$ Tokens}

\begin{table*}[!t]
\centering
\begingroup
\renewcommand{\arraystretch}{1.25}
\setlength{\tabcolsep}{4pt}
\begin{tabularx}{0.95\textwidth}{>{\raggedright\arraybackslash}p{0.14\textwidth}*{8}{>{\centering\arraybackslash}X}}
\toprule
 & \textbf{AIME24} & \textbf{AIME25} & \textbf{AIME26} & \textbf{AMC} & \textbf{MATH} & \textbf{Minerva} & \textbf{Olympiad} & \textbf{Avg.} \\
\midrule
\multicolumn{9}{c}{\textit{Student Top-$k$}} \\
TK-OPD & 19.79 & 19.17 & 14.27 & 59.04 & 85.23 & 35.29 & 52.31 & 40.73 \\
\rowcolor{blue!8}
\textbf{TT-OPD} & \textbf{26.77} & \textbf{23.33} & \textbf{22.92} & \textbf{64.87} & \textbf{86.08} & \textbf{36.81} & \textbf{54.70} & \textbf{45.07} \\
\midrule
\multicolumn{9}{c}{\textit{Teacher Top-$k$}} \\
TK-OPD & 21.56 & 17.71 & 14.90 & 57.83 & 84.75 & \textbf{35.66} & 53.76 & 40.88 \\
\rowcolor{blue!8}
\textbf{TT-OPD} & \textbf{27.08} & \textbf{25.83} & \textbf{20.21} & \textbf{64.83} & \textbf{86.58} & \textbf{35.66} & \textbf{54.70} & \textbf{44.98} \\
\midrule
\multicolumn{9}{c}{\textit{Teacher--Student Top-$k$ Overlap}} \\
TK-OPD & 20.21 & 18.33 & 15.73 & 59.30 & 85.02 & 35.02 & 53.22 & 40.98 \\
\rowcolor{blue!8}
\textbf{TT-OPD} & \textbf{26.77} & \textbf{24.17} & \textbf{23.96} & \textbf{62.69} & \textbf{86.83} & \textbf{37.50} & \textbf{55.33} & \textbf{45.32} \\
\bottomrule
\end{tabularx}
\endgroup
\caption{\textbf{Results with different selected top-$k$ tokens (accuracy, \%).} Qwen3-8B-Base-GRPO-Math is the teacher, Qwen3-4B-Base is the student, and all tested algorithms use DAPO-Math-17K as the training set. We set $k=16$. \textbf{Bold} = best among the tested algorithms within each top-$k$ block.}
\label{tab:support-selection}
\vspace{0.3em}
\begingroup
\renewcommand{\arraystretch}{1.25}
\setlength{\tabcolsep}{4pt}
\begin{tabularx}{0.95\textwidth}{>{\raggedright\arraybackslash}p{0.14\textwidth}*{8}{>{\centering\arraybackslash}X}}
\toprule
 & \textbf{AIME24} & \textbf{AIME25} & \textbf{AIME26} & \textbf{AMC} & \textbf{MATH} & \textbf{Minerva} & \textbf{Olympiad} & \textbf{Avg.} \\
\midrule
Teacher & 64.69 & 56.15 & 64.17 & 89.68 & 96.95 & 41.18 & 74.78 & 69.66 \\
\midrule
\multicolumn{9}{c}{\textit{DeepSeek-R1-Distill-Qwen-7B Student}} \\
Student & 54.58 & 40.00 & 47.50 & 82.83 & 94.27 & 37.50 & 67.15 & 60.55 \\
ST-OPD & 57.19 & 46.15 & 55.31 & 86.60 & 95.15 & 38.51 & 72.09 & 64.43 \\
TK-OPD & 58.12 & 51.15 & 61.56 & 87.50 & 95.08 & 38.47 & 73.17 & 66.44 \\
\rowcolor{blue!8}
\textbf{TT-OPD} & \textbf{63.96} & \textbf{54.90} & \textbf{62.50} & \textbf{88.55} & \textbf{96.38} & \textbf{40.49} & \textbf{73.69} & \textbf{68.64} \\
\bottomrule
\end{tabularx}
\endgroup
\caption{\textbf{Results on another model series (accuracy, \%).} Skywork-OR1-Math-7B is the teacher, DeepSeek-R1-Distill-Qwen-7B is the student, and all tested algorithms use DeepscaleR as the training set. The student's top-$16$ tokens are used as the selected top-$k$ tokens. \textbf{Bold} = best among the tested algorithms.}
\label{tab:model-series-results}
\end{table*}

\findingbox{TT-OPD consistently outperforms TK-OPD across different selected top-$k$ tokens.}
To evaluate the robustness of TT-OPD across different selected top-$k$ tokens, we use the student's top-$k$ tokens, the teacher's top-$k$ tokens, or their overlap. The accuracy and pass@$k$ results are shown in \Cref{tab:support-selection} and \Cref{tab:topk-token-sets-passk} (Appendix~\ref{sec:full-experimental-results}), respectively. TT-OPD improves average accuracy over TK-OPD by $4.10$--$4.34$ percentage points and average pass@$k$ by $5.36$--$7.88$ percentage points. These results demonstrate the robustness of TT-OPD across different constructions of the top-$k$ token set. 

\subsection{Robustness to Different Model Series}

\findingbox{TT-OPD consistently outperforms ST-OPD and TK-OPD across different model series.}

To examine the robustness of TT-OPD across different model series, we employ DeepSeek-R1-Distill-Qwen-7B~\citep{guo2025deepseekr1} as the student and Skywork-OR1-Math-7B~\citep{he2025skywork} as the teacher that is obtained by further RL post-training from DeepSeek-R1-Distill-Qwen-7B. Except for replacing the training set with DeepscaleR~\citep{tan2025deepscaler}, all configurations are identical to those described in \Cref{subssec:Experimental Setup}. The accuracy and pass@$k$ results are shown in \Cref{tab:model-series-results} and \Cref{tab:model-series-passk-results} (Appendix~\ref{sec:full-experimental-results}), respectively. TT-OPD improves average accuracy over ST-OPD and TK-OPD by $4.21$ and $2.20$ percentage points, respectively, while improving average pass@$k$ by $1.15$ and $2.76$ percentage points. Together with the Qwen3 results in \Cref{tab:main-math-results}, these results demonstrate the robustness of TT-OPD across different model series.

\begin{figure}[!t]
    \centering
    \includegraphics[width=0.98\columnwidth]{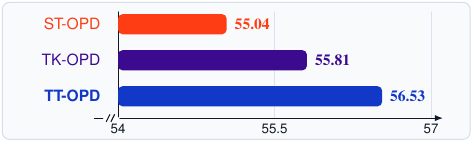}
    \vspace{-0.12cm}
    \caption{\textbf{Code-generation results (average accuracy, \%).} Qwen3-8B-Base-GRPO-Code is the teacher, Qwen3-4B-Base is the student, and all tested algorithms use the Code subset of Eurus-2-RL-Data as the training set. The student's top-$16$ tokens are used as the selected top-$k$ tokens.}
    
    \label{fig:code-generation-average}
\end{figure}

\subsection{Robustness to Different Tasks}

\findingbox{TT-OPD consistently outperforms ST-OPD and TK-OPD on code generation task.}

To test the robustness of TT-OPD across different tasks, we conduct experiments on the code-generation task. The teacher is Qwen3-8B-Base-GRPO-Code, which is trained from Qwen3-8B-Base with GRPO on 25K code-generation samples from the Code subset of Eurus-2-RL-Data~\citep{cui2025prime}, as done in ~\citet{hou2026uni}. We use Qwen3-4B-Base as the student and train on the same 25K samples for three epochs. We evaluate on HumanEval+~\citep{liu2023evalplus}, MBPP+~\citep{liu2023evalplus}, and LiveCodeBench (v6 only)~\citep{jain2024livecodebench}, as done in ~\citet{hou2026uni}. We sample four responses per prompt. All other configurations are identical to those described in \Cref{subssec:Experimental Setup}. The average accuracy results are shown in \Cref{fig:code-generation-average}, while the corresponding per-dataset accuracy and pass@4 results are reported in \Cref{tab:code-results,tab:code-passk-results} (Appendix~\ref{sec:full-experimental-results}). We observe that TT-OPD improves average accuracy over ST-OPD and TK-OPD by $1.49$ and $0.72$ percentage points, respectively, while improving average pass@4 by $1.07$ and $2.32$ percentage points. Together with the results on mathematical reasoning tasks, these results show the robustness of TT-OPD across different tasks.

\subsection{Importance of the Tail Correction}

\findingbox{Tail correction is important to TT-OPD's performance gains.}

To evaluate the importance of the tail correction, we construct TT-OPD w/o TC by removing only the tail correction term, \eg, the second term in Eq.~(\ref{eq:tt-opd-loss}). The experimental results are shown in \Cref{fig:tail-correction-ablation}, and the corresponding per-dataset accuracy and pass@$k$ results are reported in \Cref{tab:tail-correction-accuracy-full,tab:tail-correction-passk-full} (Appendix~\ref{sec:full-experimental-results}). Removing the tail correction reduces average accuracy by $5.79$ and $10.76$ percentage points for Qwen3-4B-Base and Qwen3-1.7B-Base, respectively. Relative to TK-OPD, the ablated variant is $1.45$ and $6.61$ percentage points worse for the 4B and 1.7B students, whereas TT-OPD is $4.34$ and $4.15$ percentage points better. Therefore, the tail correction is needed to account for probability mass outside the selected top-$k$ tokens. These results show that tail correction is important to TT-OPD's performance gains.

\begin{figure}[!t]
    \centering
    \includegraphics[width=0.98\columnwidth]{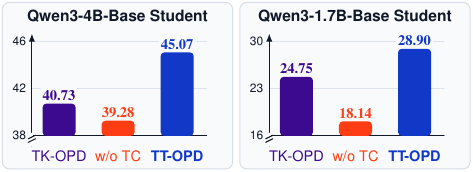}
    \caption{\textbf{Tail-correction ablation (average accuracy, \%).} Qwen3-8B-Base-GRPO-Math is the teacher, Qwen3-4B-Base and Qwen3-1.7B-Base are the students, and all tested algorithms use DAPO-Math-17K as the training set. The student's top-$16$ tokens are used as the selected top-$k$ tokens.}
    \vspace{0.115cm}
    \label{fig:tail-correction-ablation}
\end{figure}

\section{Conclusions}\label{sec:conclusion}

In this paper, we propose TT-OPD, a novel OPD variant that provides unbiased reverse KL divergence gradient estimation, rich distributional supervision, and low computational cost. To the best of our knowledge, TT-OPD is the first OPD variant to simultaneously provide all three properties. Experiments across student scales, model series, selected top-$k$ tokens, values of $k$, and tasks demonstrate that TT-OPD significantly outperforms other tested OPD variants.

\clearpage

\bibliography{references}

\clearpage
\appendix

\onecolumn

\section{Proof of Theorem \ref{thm:tt-opd-unbiased}}\label{sec:proof-tt-opd-unbiased}

\begin{proof}
We condition on a fixed student-visited state $Z_t=(x,y_{<t})$ and its selected token set $S_t^k$.
For simplicity, write $S=S_t^k$, $p(v)=p_t(v)$, $q(v)=q_t(v)$, and $A(v)=\log\frac{p(v)}{q(v)}$.
Following the stop-gradient treatment in Eq.~(\ref{eq:tt-opd-loss}), we have
\begin{equation}
\label{eq:app-opd-gradient-tt}
\begin{aligned}
    \nabla_\theta \mathcal{L}^{\mathrm{TT}}_t(\theta)
    &=
    \sum_{v\in S}A(v)\nabla_\theta p(v) \\
    &\quad+
    \mathds{1}\{y_t\notin S\}
    A(y_t)\nabla_\theta\log p(y_t).
\end{aligned}
\end{equation}
Since $y_t\sim p_t$, taking expectation over the sampled token gives
\begin{equation}
\label{eq:app-opd-gradient-expectation}
\begin{aligned}
    &\mathbb{E}_{y_t\sim p_t}
    \left[
    \nabla_\theta \mathcal{L}^{\mathrm{TT}}_t(\theta)
    \right] \\
    &=
    \sum_{v\in S}A(v)\nabla_\theta p(v)
    +
    \sum_{v\notin S}p(v)A(v)\nabla_\theta\log p(v) \\
    &=
    \sum_{v\in\mathcal{V}}A(v)\nabla_\theta p(v).
\end{aligned}
\end{equation}
On the other hand, the FV-OPD loss at this position satisfies
\begin{equation}
\label{eq:app-opd-gradient-fv}
\begin{aligned}
    \nabla_\theta \mathcal{L}^{\mathrm{FV}}_t(\theta)
    &=
    \nabla_\theta \sum_{v\in\mathcal{V}}p(v)\log\frac{p(v)}{q(v)} \\
    &=
    \sum_{v\in\mathcal{V}}\nabla_\theta p(v)
    \left(\log\frac{p(v)}{q(v)}+1\right) \\
    &=
    \sum_{v\in\mathcal{V}}A(v)\nabla_\theta p(v),
\end{aligned}
\end{equation}
where the last equality follows from $\sum_{v\in\mathcal{V}}\nabla_\theta p(v)=\nabla_\theta 1=0$.
Combining Eq.~(\ref{eq:app-opd-gradient-expectation}) and Eq.~(\ref{eq:app-opd-gradient-fv}) yields
\begin{equation}
\label{eq:app-opd-gradient-local}
\begin{aligned}
    \mathbb{E}_{y_t\sim p_t}
    \left[
    \nabla_\theta \mathcal{L}^{\mathrm{TT}}_t(\theta)
    \right]
    =
    \nabla_\theta \mathcal{L}^{\mathrm{FV}}_t(\theta).
\end{aligned}
\end{equation}
Eq.~(\ref{eq:app-opd-gradient-local}) is exactly Eq.~(\ref{eq:tt-opd-unbiased-gradient}), which proves \Cref{thm:tt-opd-unbiased}.
\end{proof}

\newpage
\section{Full Experimental Results}\label{sec:full-experimental-results}

This section provides the complete per-dataset results that supplement the main paper: pass@$k$ results for the mathematical-reasoning experiments, and both accuracy and pass@$k$ results for code generation and the tail-correction ablation.

\begin{table}[H]
\centering
\begingroup
\renewcommand{\arraystretch}{1.25}
\setlength{\tabcolsep}{4pt}
\begin{tabularx}{0.95\textwidth}{>{\raggedright\arraybackslash}p{0.14\textwidth}*{8}{>{\centering\arraybackslash}X}}
\toprule
 & \textbf{AIME24} & \textbf{AIME25} & \textbf{AIME26} & \textbf{AMC} & \textbf{MATH} & \textbf{Minerva} & \textbf{Olympiad} & \textbf{Avg.} \\
\midrule
Teacher & 53.33 & 40.00 & 40.00 & 86.75 & 93.20 & 46.32 & 70.81 & 61.49 \\
\midrule
\multicolumn{9}{c}{\textit{Qwen3-4B-Base Student}} \\
Student & 26.67 & 26.67 & 20.00 & 63.86 & 76.20 & 36.03 & 42.37 & 41.69 \\
ST-OPD & 40.00 & 40.00 & 33.33 & 79.52 & 90.40 & \textbf{42.65} & 63.70 & 55.66 \\
TK-OPD & 43.33 & 26.67 & 26.67 & 71.08 & 88.60 & 41.91 & \textbf{65.63} & 51.98 \\
\rowcolor{blue!8}
\textbf{TT-OPD} & \textbf{56.67} & \textbf{43.33} & \textbf{43.33} & \textbf{80.72} & \textbf{92.20} & 37.87 & 64.89 & \textbf{59.86} \\
\midrule
\multicolumn{9}{c}{\textit{Qwen3-1.7B-Base Student}} \\
Student & 20.00 & 10.00 & 10.00 & 53.01 & 70.80 & 28.31 & 35.70 & 32.55 \\
ST-OPD & \textbf{30.00} & 10.00 & 6.67 & \textbf{59.04} & 76.00 & 27.57 & \textbf{43.56} & 36.12 \\
TK-OPD & 23.33 & 10.00 & 16.67 & 54.22 & \textbf{76.80} & \textbf{29.78} & 41.33 & 36.02 \\
\rowcolor{blue!8}
\textbf{TT-OPD} & 26.67 & \textbf{20.00} & \textbf{23.33} & \textbf{59.04} & 70.80 & 26.47 & 37.93 & \textbf{37.75} \\
\bottomrule
\end{tabularx}
\endgroup
\caption{\textbf{Pass@$k$ results on mathematical reasoning (\%).} Unlike average accuracy, which averages correctness over all sampled responses, pass@$k$ is the percentage of problems for which at least one of the $k$ sampled responses is correct. We report pass@32 on AIME24, AIME25, AIME26, and AMC, and pass@8 on MATH, Minerva, and Olympiad. All other settings are identical to those in \Cref{tab:main-math-results}. \textbf{Bold} = best among the tested algorithms within each student block.}
\label{tab:main-math-passk-results}
\end{table}

\begin{table}[H]
\centering
\begingroup
\renewcommand{\arraystretch}{1.25}
\setlength{\tabcolsep}{4pt}
\begin{tabularx}{0.95\textwidth}{>{\raggedright\arraybackslash}p{0.14\textwidth}*{8}{>{\centering\arraybackslash}X}}
\toprule
 & \textbf{AIME24} & \textbf{AIME25} & \textbf{AIME26} & \textbf{AMC} & \textbf{MATH} & \textbf{Minerva} & \textbf{Olympiad} & \textbf{Avg.} \\
\midrule
ST-OPD & 40.00 & 40.00 & 33.33 & 79.52 & 90.40 & 42.65 & 63.70 & 55.66 \\
\midrule
\multicolumn{9}{c}{\textit{$k=8$}} \\
TK-OPD & 40.00 & 23.33 & 26.67 & 73.49 & 90.60 & \textbf{41.18} & \textbf{62.22} & 51.07 \\
\rowcolor{blue!8}
\textbf{TT-OPD} & \textbf{56.67} & \textbf{36.67} & \textbf{43.33} & \textbf{81.93} & \textbf{92.40} & 37.50 & 57.78 & \textbf{58.04} \\
\midrule
\multicolumn{9}{c}{\textit{$k=16$}} \\
TK-OPD & 43.33 & 26.67 & 26.67 & 71.08 & 88.60 & \textbf{41.91} & \textbf{65.63} & 51.98 \\
\rowcolor{blue!8}
\textbf{TT-OPD} & \textbf{56.67} & \textbf{43.33} & \textbf{43.33} & \textbf{80.72} & \textbf{92.20} & 37.87 & 64.89 & \textbf{59.86} \\
\midrule
\multicolumn{9}{c}{\textit{$k=32$}} \\
TK-OPD & 40.00 & 33.33 & 30.00 & 75.90 & \textbf{91.60} & 41.91 & 60.74 & 53.35 \\
\rowcolor{blue!8}
\textbf{TT-OPD} & \textbf{56.67} & \textbf{36.67} & \textbf{50.00} & \textbf{80.72} & 91.40 & \textbf{43.38} & \textbf{66.07} & \textbf{60.70} \\
\midrule
\multicolumn{9}{c}{\textit{$k=64$}} \\
TK-OPD & 46.67 & 30.00 & 36.67 & 80.72 & 91.20 & 42.28 & 62.07 & 55.66 \\
\rowcolor{blue!8}
\textbf{TT-OPD} & \textbf{56.67} & \textbf{53.33} & \textbf{50.00} & \textbf{87.95} & \textbf{94.00} & \textbf{43.01} & \textbf{65.63} & \textbf{64.37} \\
\bottomrule
\end{tabularx}
\endgroup
\caption{\textbf{Pass@$k$ results for different numbers of selected tokens (\%).} We report pass@32 on AIME24, AIME25, AIME26, and AMC, and pass@8 on MATH, Minerva, and Olympiad. All other settings are identical to those in \Cref{tab:k-sensitivity}. ST-OPD does not depend on $k$. \textbf{Bold} = best between TK-OPD and TT-OPD within each $k$ block.}
\label{tab:k-sensitivity-passk}
\end{table}

\begin{table}[!t]
\centering
\begingroup
\renewcommand{\arraystretch}{1.25}
\setlength{\tabcolsep}{4pt}
\begin{tabularx}{0.95\textwidth}{>{\raggedright\arraybackslash}p{0.14\textwidth}*{8}{>{\centering\arraybackslash}X}}
\toprule
 & \textbf{AIME24} & \textbf{AIME25} & \textbf{AIME26} & \textbf{AMC} & \textbf{MATH} & \textbf{Minerva} & \textbf{Olympiad} & \textbf{Avg.} \\
\midrule
\multicolumn{9}{c}{\textit{Student Top-$k$}} \\
TK-OPD & 43.33 & 26.67 & 26.67 & 71.08 & 88.60 & \textbf{41.91} & \textbf{65.63} & 51.98 \\
\rowcolor{blue!8}
\textbf{TT-OPD} & \textbf{56.67} & \textbf{43.33} & \textbf{43.33} & \textbf{80.72} & \textbf{92.20} & 37.87 & 64.89 & \textbf{59.86} \\
\midrule
\multicolumn{9}{c}{\textit{Teacher Top-$k$}} \\
TK-OPD & \textbf{46.67} & 30.00 & 30.00 & 72.29 & 88.20 & 37.87 & 62.37 & 52.49 \\
\rowcolor{blue!8}
\textbf{TT-OPD} & 43.33 & \textbf{33.33} & \textbf{43.33} & \textbf{81.93} & \textbf{92.40} & \textbf{44.85} & \textbf{65.78} & \textbf{57.85} \\
\midrule
\multicolumn{9}{c}{\textit{Teacher--Student Top-$k$ Overlap}} \\
TK-OPD & 40.00 & 26.67 & 30.00 & 77.11 & 89.00 & \textbf{44.49} & 61.33 & 52.66 \\
\rowcolor{blue!8}
\textbf{TT-OPD} & \textbf{50.00} & \textbf{46.67} & \textbf{40.00} & \textbf{80.72} & \textbf{90.80} & 40.44 & \textbf{63.26} & \textbf{58.84} \\
\bottomrule
\end{tabularx}
\endgroup
\caption{\textbf{Pass@$k$ results with different selected top-$k$ tokens (\%).} We report pass@32 on AIME24, AIME25, AIME26, and AMC, and pass@8 on MATH, Minerva, and Olympiad. All other settings are identical to those in \Cref{tab:support-selection}. \textbf{Bold} = best among the tested algorithms within each top-$k$ block.}
\label{tab:topk-token-sets-passk}
\end{table}

\begin{table}[!t]
\centering
\begingroup
\renewcommand{\arraystretch}{1.25}
\setlength{\tabcolsep}{4pt}
\begin{tabularx}{0.95\textwidth}{>{\raggedright\arraybackslash}p{0.14\textwidth}*{8}{>{\centering\arraybackslash}X}}
\toprule
 & \textbf{AIME24} & \textbf{AIME25} & \textbf{AIME26} & \textbf{AMC} & \textbf{MATH} & \textbf{Minerva} & \textbf{Olympiad} & \textbf{Avg.} \\
\midrule
Teacher & 86.67 & 80.00 & 80.00 & 100.00 & 99.00 & 47.06 & 82.37 & 82.16 \\
\midrule
\multicolumn{9}{c}{\textit{DeepSeek-R1-Distill-Qwen-7B Student}} \\
Student & 83.33 & 73.33 & 76.67 & 96.39 & 98.00 & 44.85 & 82.52 & 79.30 \\
ST-OPD & \textbf{83.33} & 76.67 & \textbf{80.00} & 97.59 & \textbf{98.80} & 47.06 & 82.22 & 80.81 \\
TK-OPD & 80.00 & 70.00 & 76.67 & 98.80 & 98.60 & 47.06 & 83.26 & 79.20 \\
\rowcolor{blue!8}
\textbf{TT-OPD} & \textbf{83.33} & \textbf{80.00} & \textbf{80.00} & \textbf{100.00} & 98.40 & \textbf{48.16} & \textbf{83.85} & \textbf{81.96} \\
\bottomrule
\end{tabularx}
\endgroup
\caption{\textbf{Pass@$k$ results on another model series (\%).} We report pass@32 on AIME24, AIME25, AIME26, and AMC, and pass@8 on MATH, Minerva, and Olympiad. All other settings are identical to those in \Cref{tab:model-series-results}. \textbf{Bold} = best among the tested algorithms.}
\label{tab:model-series-passk-results}
\end{table}

\begin{table}[!t]
\centering
\begingroup
\renewcommand{\arraystretch}{1.25}
\setlength{\tabcolsep}{4pt}
\begin{tabularx}{0.95\textwidth}{>{\raggedright\arraybackslash}p{0.14\textwidth}*{4}{>{\centering\arraybackslash}X}}
\toprule
 & \textbf{HumanEval+} & \textbf{MBPP+} & \textbf{LiveCodeBench} & \textbf{Avg.} \\
\midrule
Teacher & 77.44 & 71.96 & 27.43 & 58.94 \\
\midrule
\multicolumn{5}{c}{\textit{Qwen3-4B-Base Student}} \\
Student & 18.14 & 13.16 & 16.14 & 15.81 \\
ST-OPD & 72.10 & 68.72 & 24.29 & 55.04 \\
TK-OPD & 74.70 & 68.72 & 24.00 & 55.81 \\
\rowcolor{blue!8}
\textbf{TT-OPD} & \textbf{75.00} & \textbf{69.44} & \textbf{25.14} & \textbf{56.53} \\
\bottomrule
\end{tabularx}
\endgroup
\caption{\textbf{Per-dataset code-generation results (accuracy, \%).} Qwen3-8B-Base-GRPO-Code is the teacher, Qwen3-4B-Base is the student, and all tested algorithms use the Code subset of Eurus-2-RL-Data as the training set. The student's top-$16$ tokens are used as the selected token set. \textbf{Bold} = best among the tested algorithms.}
\label{tab:code-results}
\end{table}

\begin{table}[!t]
\centering
\begingroup
\renewcommand{\arraystretch}{1.25}
\setlength{\tabcolsep}{4pt}
\begin{tabularx}{0.95\textwidth}{>{\raggedright\arraybackslash}p{0.14\textwidth}*{4}{>{\centering\arraybackslash}X}}
\toprule
 & \textbf{HumanEval+} & \textbf{MBPP+} & \textbf{LiveCodeBench} & \textbf{Avg.} \\
\midrule
Teacher & 87.20 & 80.42 & 31.43 & 66.35 \\
\midrule
\multicolumn{5}{c}{\textit{Qwen3-4B-Base Student}} \\
Student & 48.78 & 41.27 & 28.00 & 39.35 \\
ST-OPD & 84.15 & 78.84 & 28.00 & 63.66 \\
TK-OPD & 81.71 & 77.51 & 28.00 & 62.41 \\
\rowcolor{blue!8}
\textbf{TT-OPD} & \textbf{85.37} & \textbf{79.10} & \textbf{29.71} & \textbf{64.73} \\
\bottomrule
\end{tabularx}
\endgroup
\caption{\textbf{Per-dataset code-generation results (pass@4, \%).} All other settings are identical to those in \Cref{tab:code-results}. \textbf{Bold} = best among the tested algorithms.}
\label{tab:code-passk-results}
\end{table}

\begin{table}[!t]
\centering
\begingroup
\renewcommand{\arraystretch}{1.18}
\setlength{\tabcolsep}{4pt}
\begin{tabularx}{0.95\textwidth}{>{\raggedright\arraybackslash}p{0.14\textwidth}*{8}{>{\centering\arraybackslash}X}}
\toprule
 & \textbf{AIME24} & \textbf{AIME25} & \textbf{AIME26} & \textbf{AMC} & \textbf{MATH} & \textbf{Minerva} & \textbf{Olympiad} & \textbf{Avg.} \\
\midrule
Teacher & 31.15 & 24.69 & 22.71 & 70.63 & 88.02 & 38.10 & 58.91 & 47.74 \\
\midrule
\multicolumn{9}{c}{\textit{Qwen3-4B-Base Student}} \\
Initial Student & 8.02 & 5.31 & 6.15 & 30.80 & 54.07 & 22.79 & 25.22 & 21.77 \\
TK-OPD & 19.79 & 19.17 & 14.27 & 59.04 & 85.23 & 35.29 & 52.31 & 40.73 \\
TT-OPD w/o TC & 21.56 & 20.63 & 19.58 & 57.27 & 72.70 & 33.92 & 49.30 & 39.28 \\
\rowcolor{blue!8}
\textbf{TT-OPD} & \textbf{26.77} & \textbf{23.33} & \textbf{22.92} & \textbf{64.87} & \textbf{86.08} & \textbf{36.81} & \textbf{54.70} & \textbf{45.07} \\
\midrule
\multicolumn{9}{c}{\textit{Qwen3-1.7B-Base Student}} \\
Initial Student & 3.65 & 1.04 & 2.81 & 25.34 & 52.45 & 15.99 & 22.93 & 17.74 \\
TK-OPD & 8.54 & 2.40 & 4.90 & 35.54 & 68.00 & 22.47 & 31.39 & 24.75 \\
TT-OPD w/o TC & 4.69 & 1.88 & 3.02 & 20.75 & 56.93 & 12.50 & 27.22 & 18.14 \\
\rowcolor{blue!8}
\textbf{TT-OPD} & \textbf{14.17} & \textbf{6.98} & \textbf{8.33} & \textbf{42.13} & \textbf{69.60} & \textbf{24.63} & \textbf{36.48} & \textbf{28.90} \\
\bottomrule
\end{tabularx}
\endgroup
\caption{\textbf{Per-dataset tail-correction ablation results (accuracy, \%).} Qwen3-8B-Base-GRPO-Math is the teacher, Qwen3-4B-Base and Qwen3-1.7B-Base are the students, and all tested algorithms use DAPO-Math-17K as the training set. The student's top-$16$ tokens are used as the selected token set. ``w/o TC'' removes only the tail-correction term, i.e., the second term in Eq.~(\ref{eq:tt-opd-loss}). \textbf{Bold} = best among the tested algorithms within each student block.}
\label{tab:tail-correction-accuracy-full}
\end{table}

\begin{table}[!t]
\centering
\begingroup
\renewcommand{\arraystretch}{1.18}
\setlength{\tabcolsep}{4pt}
\begin{tabularx}{0.95\textwidth}{>{\raggedright\arraybackslash}p{0.14\textwidth}*{8}{>{\centering\arraybackslash}X}}
\toprule
 & \textbf{AIME24} & \textbf{AIME25} & \textbf{AIME26} & \textbf{AMC} & \textbf{MATH} & \textbf{Minerva} & \textbf{Olympiad} & \textbf{Avg.} \\
\midrule
Teacher & 53.33 & 40.00 & 40.00 & 86.75 & 93.20 & 46.32 & 70.81 & 61.49 \\
\midrule
\multicolumn{9}{c}{\textit{Qwen3-4B-Base Student}} \\
Initial Student & 26.67 & 26.67 & 20.00 & 63.86 & 76.20 & 36.03 & 42.37 & 41.69 \\
TK-OPD & 43.33 & 26.67 & 26.67 & 71.08 & 88.60 & \textbf{41.91} & \textbf{65.63} & 51.98 \\
TT-OPD w/o TC & 33.33 & 26.67 & 30.00 & 74.70 & 89.40 & 38.24 & 64.15 & 50.93 \\
\rowcolor{blue!8}
\textbf{TT-OPD} & \textbf{56.67} & \textbf{43.33} & \textbf{43.33} & \textbf{80.72} & \textbf{92.20} & 37.87 & 64.89 & \textbf{59.86} \\
\midrule
\multicolumn{9}{c}{\textit{Qwen3-1.7B-Base Student}} \\
Initial Student & 20.00 & 10.00 & 10.00 & 53.01 & 70.80 & 28.31 & 35.70 & 32.55 \\
TK-OPD & 23.33 & 10.00 & 16.67 & 54.22 & \textbf{76.80} & \textbf{29.78} & \textbf{41.33} & 36.02 \\
TT-OPD w/o TC & 20.00 & 10.00 & 10.00 & 53.01 & 76.00 & 27.21 & 41.04 & 33.89 \\
\rowcolor{blue!8}
\textbf{TT-OPD} & \textbf{26.67} & \textbf{20.00} & \textbf{23.33} & \textbf{59.04} & 70.80 & 26.47 & 37.93 & \textbf{37.75} \\
\bottomrule
\end{tabularx}
\endgroup
\caption{\textbf{Per-dataset tail-correction ablation results (pass@$k$, \%).} We report pass@32 on AIME24, AIME25, AIME26, and AMC, and pass@8 on MATH, Minerva, and Olympiad. All other settings are identical to those in \Cref{tab:tail-correction-accuracy-full}. \textbf{Bold} = best among the tested algorithms within each student block.}
\label{tab:tail-correction-passk-full}
\end{table}

\definecolor{casequestionframe}{RGB}{95,105,120}
\definecolor{casequestionback}{RGB}{247,248,250}
\definecolor{casestframe}{RGB}{171,55,55}
\definecolor{casestback}{RGB}{255,247,247}
\definecolor{casetkframe}{RGB}{178,105,24}
\definecolor{casetkback}{RGB}{255,250,242}
\definecolor{casettframe}{RGB}{38,91,154}
\definecolor{casettback}{RGB}{244,248,255}
\definecolor{caseanalysisframe}{RGB}{74,116,84}
\definecolor{caseanalysisback}{RGB}{246,251,247}

\newcommand{\casebody}{
  \scriptsize
  \setlength{\abovedisplayskip}{3pt}
  \setlength{\belowdisplayskip}{3pt}
  \setlength{\abovedisplayshortskip}{2pt}
  \setlength{\belowdisplayshortskip}{2pt}
}
\newcommand{\caseomitted}{
  \par\smallskip
  \noindent\textcolor{casequestionframe}{
    \textit{\ldots\ continued reasoning omitted \ldots}}
  \par\smallskip
}

\newtcolorbox{casequestionbox}{
  colback=casequestionback,
  colframe=casequestionframe,
  boxrule=0.6pt,
  arc=1.2mm,
  left=2mm,
  right=2mm,
  top=1.2mm,
  bottom=1.2mm
}

\newtcolorbox{casestbox}{
  colback=casestback,
  colframe=casestframe,
  coltitle=black,
  fonttitle=\small\bfseries,
  halign title=center,
  title={ST-OPD (Incorrect)},
  boxrule=0.7pt,
  arc=1.2mm,
  left=1.7mm,
  right=1.7mm,
  top=1.5mm,
  bottom=1.5mm,
  height=0.62\textheight,
  valign=top
}

\newtcolorbox{casetkbox}{
  colback=casetkback,
  colframe=casetkframe,
  coltitle=black,
  fonttitle=\small\bfseries,
  halign title=center,
  title={TK-OPD (Incorrect)},
  boxrule=0.7pt,
  arc=1.2mm,
  left=1.7mm,
  right=1.7mm,
  top=1.5mm,
  bottom=1.5mm,
  height=0.62\textheight,
  valign=top
}

\newtcolorbox{casettbox}{
  colback=casettback,
  colframe=casettframe,
  coltitle=black,
  fonttitle=\small\bfseries,
  halign title=center,
  title={TT-OPD (Correct)},
  boxrule=0.8pt,
  arc=1.2mm,
  left=1.7mm,
  right=1.7mm,
  top=1.5mm,
  bottom=1.5mm,
  height=0.62\textheight,
  valign=top
}

\newtcolorbox{caseanalysisbox}{
  colback=caseanalysisback,
  colframe=caseanalysisframe,
  boxrule=0.6pt,
  arc=1.2mm,
  left=2mm,
  right=2mm,
  top=1mm,
  bottom=1mm
}

\clearpage
\section{Case Studies}\label{sec:case-studies}

We present three AIME-2024 examples for which TT-OPD produces the
correct answer while ST-OPD and TK-OPD both fail. Each page compares
one rollout from the three algorithms under the same student model.
To preserve the character of the original generations, the lightly
typeset excerpts retain first-person deliberation, reversals, and
verification steps. Ellipses mark only repetitive or space-limited
passages; the displayed final answers follow the original rollouts.

\subsection{Qwen3-4B-Base Student}\label{sec:case-study-qwen4b}

\begin{casequestionbox}
\small
\textbf{Question.}
Find the largest possible real part of
\[
  (75+117i)z+\frac{96+144i}{z},
  \qquad |z|=4.
\]
\textbf{Reference answer:} \(540\).
\end{casequestionbox}

\noindent
\begin{minipage}[t]{0.318\textwidth}
\begin{casestbox}
\casebody
\textbf{Model rollout (excerpt).}
``Alright, I have this problem to solve. First, I need to understand
what's being asked.'' It sets
\[
  z=4e^{i\theta},\qquad
  \frac1z=\frac14e^{-i\theta},
\]
but then rewrites the reciprocal coefficient as
\[
  \frac{96+144i}{z}
  =24(1+2i)e^{-i\theta}.
\]
``Wait, perhaps there's a better way. Maybe I can express \(z\) in
terms of its real and imaginary parts.'' After briefly trying
\(z=x+iy\), it says, ``This seems complicated. Maybe I can use the
polar form after all.''

\caseomitted

Expanding the two terms, it writes
\[
  \begin{aligned}
  \operatorname{Re}(T_1)&=300\cos\theta-468\sin\theta,\\
  \operatorname{Re}(T_2)&=24\cos\theta+48\sin\theta.
  \end{aligned}
\]
It nevertheless combines them as
\[
  \operatorname{Re}(T_1+T_2)
  =324\cos\theta-468\sin\theta
\]
and obtains
\(\sqrt{324^2+468^2}=180\sqrt{10}\).
After rechecking the square root, the rollout repeats the same answer
several times.

\medskip
\textbf{Model answer:}
\[
  \boxed{180\sqrt{10}}
\]
\end{casestbox}
\end{minipage}\hfill
\begin{minipage}[t]{0.318\textwidth}
\begin{casetkbox}
\casebody
\textbf{Model rollout (excerpt).}
``Let's compute each part separately.'' With
\(z=4e^{i\theta}\), it writes
\[
  (75+117i)z=4(75+117i)e^{i\theta}
\]
and
\[
  \frac{96+144i}{z}=24(2+3i)e^{-i\theta}.
\]
It denotes \(A=75+117i\) and \(B=2+3i\), giving
\[
  4A(\cos\theta+i\sin\theta)
  +24Be^{-i\theta}.
\]
The response then claims, ``The real part here is
\(4A\cos\theta\),'' and similarly takes the second real contribution
to be \(24B\cos\theta\).

\caseomitted

Thus it obtains
\[
  (4A+24B)\cos\theta
  =(348+540i)\cos\theta.
\]
``Wait, that doesn't seem right.'' It then discards the imaginary
coefficient, concludes that the real part is \(348\cos\theta\), and
maximizes it at \(\cos\theta=1\).

\medskip
\textbf{Model answer:}
\[
  \boxed{348}
\]
\end{casetkbox}
\end{minipage}\hfill
\begin{minipage}[t]{0.318\textwidth}
\begin{casettbox}
\casebody
\textbf{Model rollout (excerpt).}
``Let's compute each part separately.'' The rollout sets
\(z=4e^{i\theta}\) and correctly keeps
\[
  \frac{96+144i}{4e^{i\theta}}
  =(24+36i)e^{-i\theta}.
\]
Writing \(w=e^{i\theta}\), it expands
\[
  300w+468iw+24\overline w+36i\overline w.
\]
It then checks the four contributions one at a time:
\[
  \begin{aligned}
  300w&=300\cos\theta+300i\sin\theta,\\
  468iw&=-468\sin\theta+468i\cos\theta,\\
  24\overline w&=24\cos\theta-24i\sin\theta,\\
  36i\overline w&=36\sin\theta+36i\cos\theta.
  \end{aligned}
\]
After combining like terms, the real part is
\[
  324\cos\theta-432\sin\theta.
\]
``Recall that \(A\cos\theta+B\sin\theta\) has maximum
\(\sqrt{A^2+B^2}\).'' Therefore,
\[
  \sqrt{324^2+432^2}=\sqrt{291600}=540.
\]

\medskip
\textbf{Model answer:}
\[
  \boxed{540}
\]
\end{casettbox}
\end{minipage}

\begin{caseanalysisbox}
\small
\textbf{Comparison.}
ST-OPD changes \((96+144i)/4\) into \(24(1+2i)\), and TK-OPD
incorrectly takes the real part of expressions with complex
coefficients. TT-OPD preserves both sine and cosine contributions and
therefore reaches the reference answer.
\end{caseanalysisbox}

\newpage

\vspace*{6pt}

\subsection{Qwen3-1.7B-Base Student}\label{sec:case-study-qwen17b}

\begin{casequestionbox}
\small
\textbf{Question.}
Aya walks \(9\) kilometers and then spends \(t\) minutes in a coffee
shop. At \(s\) km/h, the total time is \(4\) hours; at \(s+2\) km/h,
it is \(2\) hours \(24\) minutes. How many minutes does the trip take
at \(s+\tfrac12\) km/h, including the same coffee-shop time?

\textbf{Reference answer:} \(204\).
\end{casequestionbox}

\noindent
\begin{minipage}[t]{0.318\textwidth}
\begin{casestbox}
\casebody
\textbf{Model rollout (excerpt).}
``Alright, I have this problem about Aya's morning walk. Let's break
it down step by step.'' It correctly converts \(2\) hours \(24\)
minutes to \(2.4\) hours, but then treats \(t\) itself as hours:
\[
  \frac{9}{s}+t=4,
  \qquad
  \frac{9}{s+2}+t=2.4.
\]
Subtracting gives
\[
  \frac{18}{s(s+2)}=1.6,
  \qquad s^2+2s-11.25=0.
\]
The quadratic formula yields \(s=2.5\) or \(-4.5\), and it rejects the
negative speed.

\caseomitted

Substituting into its first equation, it writes
\[
  \frac{9}{2.5}+t=4,\qquad 3.6+t=4,
\]
then states \(t=0.4\) ``minutes.'' For the requested speed,
\[
  s+\frac12=3,\qquad \frac93=3\text{ hours}.
\]
It converts only this walking time to \(180\) minutes and nevertheless
claims that the result includes the coffee-shop stop.

\medskip
\textbf{Model answer:}
\[
  \boxed{180}
\]
\end{casestbox}
\end{minipage}\hfill
\begin{minipage}[t]{0.318\textwidth}
\begin{casetkbox}
\casebody
\textbf{Model rollout (excerpt).}
``First, I notice that the total times are given in hours and
minutes.'' Unlike ST-OPD, it converts the stop consistently:
\[
  \frac{9}{s}+\frac{t}{60}=4,
  \qquad
  \frac{9}{s+2}+\frac{t}{60}=2.4.
\]
Eliminating \(t/60\), it derives
\[
  \frac{18}{s(s+2)}=1.6,
  \qquad s^2+2s-11.25=0,
\]
and selects \(s=2.5\). It then computes
\[
  3.6+\frac{t}{60}=4
  \quad\Longrightarrow\quad t=24.
\]

\caseomitted

At \(s+\tfrac12=3\), the response obtains
\[
  \frac93+\frac{24}{60}=3.4\text{ hours}
  =3\text{ hours }24\text{ minutes}.
\]
Its prose conclusion is therefore correct. However, after the heading
``Final Answer,'' the emitted answer abruptly changes to \(32\), with
no supporting calculation.

\medskip
\textbf{Model answer:}
\[
  \boxed{32}
\]
\end{casetkbox}
\end{minipage}\hfill
\begin{minipage}[t]{0.318\textwidth}
\begin{casettbox}
\casebody
\textbf{Model rollout (excerpt).}
``We'll set up equations based on the first two scenarios and solve
for \(s\) and \(t\).'' It first records
\[
  t\text{ minutes}=\frac{t}{60}\text{ hours},
\]
and consequently uses
\[
  \frac{9}{s}+\frac{t}{60}=4,
  \qquad
  \frac{9}{s+2}+\frac{t}{60}=2.4.
\]
Subtracting the equations gives
\[
  \frac{18}{s(s+2)}=1.6,
  \qquad s^2+2s-11.25=0.
\]
It explicitly checks both quadratic roots and keeps \(s=2.5\) because
speed cannot be negative.

\caseomitted

Returning to the first equation,
\[
  \frac{9}{2.5}+\frac{t}{60}=4
  \quad\Longrightarrow\quad t=24.
\]
Finally,
\[
  \frac{9}{2.5+\tfrac12}+\frac{24}{60}
  =3.4\text{ hours},
\]
and the rollout performs the requested unit conversion:
\[
  3.4\times60=204\text{ minutes}.
\]

\medskip
\textbf{Model answer:}
\[
  \boxed{204}
\]
\end{casettbox}
\end{minipage}

\begin{caseanalysisbox}
\small
\textbf{Comparison.}
ST-OPD mixes minutes and hours and drops the coffee-shop time.
TK-OPD reaches the correct intermediate result but emits an unrelated
final answer. TT-OPD maintains unit consistency and agreement between
its derivation and final answer.
\end{caseanalysisbox}

\newpage
\vspace*{12pt}

\subsection{DeepSeek-R1-Distill-Qwen-7B Student}
\label{sec:case-study-deepseek7b}

\begin{casequestionbox}
\small
\textbf{Question.}
Let \(p\) be the least prime for which \(p^2\mid n^4+1\) for some
positive integer \(n\). Find the least positive integer \(m\) such
that \(p^2\mid m^4+1\).

\textbf{Reference answer:} \(110\).
\end{casequestionbox}

\noindent
\begin{minipage}[t]{0.318\textwidth}
\begin{casestbox}
\casebody
\textbf{Model rollout (excerpt).}
After testing smaller primes, the response says, ``Wait, perhaps
\(p=17\).'' It observes
\[
  2^4=16\equiv-1\pmod{17}
\]
and checks \(f'(2)=4\cdot2^3=32\not\equiv0\pmod{17}\), so Hensel's
lemma guarantees a lift.

\caseomitted

It considers only \(x\equiv2\pmod{17}\). With \(x=2+17k\), the
rollout obtains
\[
  17+255k\equiv0\pmod{289},
\]
then divides by \(17\):
\[
  15k\equiv16\pmod{17}.
\]
Hence \(k\equiv9\pmod{17}\), so \(x=2+17\cdot9=155\). The response
repeatedly asks, ``Wait, is
\(155\) the least positive integer \(m\)?'' It checks
\(138,121,104,87,\ldots\), all from the same residue class, and then
concludes: ``The minimal solution mod \(289\) is \(m=155\), because
the lift from \(2\bmod17\) is unique.'' It never enumerates the other
three roots modulo \(17\).

\medskip
\textbf{Model answer:}
\[
  \boxed{155}
\]
\end{casestbox}
\end{minipage}\hfill
\begin{minipage}[t]{0.318\textwidth}
\begin{casetkbox}
\casebody
\textbf{Model rollout (excerpt).}
The rollout first follows the \(x\equiv2\pmod{17}\) branch and obtains
\(x=155\). It then pauses: ``Wait, perhaps there is another solution
where \(x\equiv-2\pmod{17}\).'' Taking \(x=15+17k\), it expands
\[
  50+34k\equiv288\pmod{289}.
\]
After dividing by \(17\), it solves
\[
  2k\equiv14\pmod{17}.
\]
Thus \(k\equiv7\pmod{17}\), giving \(x=15+17\cdot7=134\).

\caseomitted

The response directly verifies the new candidate:
\[
  134^2\equiv38\pmod{289},
\]
and
\[
  134^4=38^2=1444\equiv-1\pmod{289}.
\]
It then states, ``The solutions mod \(289\) are \(155\) and \(134\).
So the minimal positive solution is \(134\).'' This improves upon the
first branch but still omits the roots \(8\) and \(9\) modulo \(17\).
Its final summary reports both \(p\) and \(m\).

\medskip
\textbf{Model answer:}
\[
  \boxed{17,\;134}
\]
\end{casetkbox}
\end{minipage}\hfill
\begin{minipage}[t]{0.318\textwidth}
\begin{casettbox}
\casebody
\textbf{Model rollout (excerpt).}
The response initially finds the lift \(2\mapsto155\), but then asks,
``Perhaps there is a smaller \(m\) that is not \(2\bmod17\).'' It
enumerates every root:
\[
  x\equiv2,8,9,15\pmod{17}.
\]
For the \(x\equiv8\) branch, write \(x=8+17k\). The rollout computes
\[
  (8+17k)^4+1\equiv51+136k\pmod{289},
\]
so \(51+136k\equiv0\pmod{289}\).
Dividing by \(17\) gives
\[
  8k\equiv14\pmod{17}.
\]
Hence \(k\equiv6\pmod{17}\), so \(x=8+17\cdot6=110\). It verifies
\[
  110^2\equiv251\pmod{289},
\]
and
\[
  110^4=251^2=63001\equiv-1\pmod{289}.
\]

\caseomitted

After checking the remaining residue classes, it summarizes
\[
\begin{array}{c@{\ \mapsto\ }c@{\qquad}c@{\ \mapsto\ }c}
2 & 155 & 8 & 110\\
9 & 179 & 15 & 134
\end{array}
\]
and concludes, ``Among all four cases, the minimal \(m\) is \(110\).''

\medskip
\textbf{Model answer:}
\[
  \boxed{110}
\]
\end{casettbox}
\end{minipage}

\begin{caseanalysisbox}
\small
\textbf{Comparison.}
The incorrect answers arise from incomplete branch enumeration:
ST-OPD lifts one root and TK-OPD lifts two. TT-OPD checks all four
roots modulo \(17\), so its comparison of the lifted solutions
correctly identifies the global minimum.
\end{caseanalysisbox}

\end{document}